\documentclass[journal]{IEEEtran}
\ifCLASSINFOpdf
\else
\fi
\usepackage{graphicx}
\usepackage{textcomp}
\usepackage{xcolor}

\usepackage{adjustbox}
\usepackage{gensymb}
\usepackage{times}
\usepackage{fancyvrb}
\usepackage{cite}
\usepackage{multicol}
\usepackage{epsfig}
\usepackage{longtable}
\usepackage{multirow}
\usepackage{tabularx}
\usepackage{csquotes}
\usepackage{color}
\usepackage{amsfonts,amsmath}
\usepackage{times,helvet,courier}
\usepackage{latexsym,amssymb}
\usepackage{xspace,epsfig}
\usepackage[noend]{algpseudocode}
\usepackage{algorithm}
\usepackage{setspace}
\usepackage{rotating}
\usepackage{pgf,tikz}
\usepackage{varwidth}
\usepackage[noend]{algpseudocode}
\usepackage{float}
\usepackage{varwidth}
\usepackage{lscape}
\usepackage{lipsum}
\usepackage{multirow}
\usepackage{dblfloatfix}

\usepackage{amssymb}
\usepackage{ dsfont }
\usepackage{subcaption}

\DeclareMathOperator*{\res}{Res}

\usepackage{amsthm}
\newtheorem{prop}{\noindent \textbf{Proposition}}
\newtheorem{lemma}{\noindent \textbf{Lemma}}
\newtheorem*{remark}{\noindent \textbf{Remarks}}

\begin{document}
	
	\title{\LARGE \bf
		Necessary and Sufficient Conditions for Passivity of\\ Velocity-Sourced Impedance Control of Series Elastic Actuators
	}
	
	\author{Fatih Emre Tosun \and Volkan Patoglu
		\thanks{F. E. Tosun and V. Patoglu are with Faculty of Engineering and Natural Sciences, Sabanc{\i}
			University, \.Istanbul, Turkey {\tt
				\{femretosun,vpatoglu\}@sabanciuniv.edu} }}
	
	\author{
		Fatih Emre Tosun,~\IEEEmembership{Student Member,~IEEE}
		and~Volkan Patoglu,~\IEEEmembership{Member,~IEEE}
		\thanks{F. E. Tosun and V. Patoglu are with the Faculty of Engineering and Natural Sciences, Sabanc{\i}
			University, \.Istanbul, Turkey.}%
		\thanks{Manuscript received February 2, 2019.}
	}


	\maketitle
	
	\begin{abstract}
		Series Elastic Actuation (SEA) has become prevalent in applications involving physical human-robot interaction as it provides considerable advantages over traditional stiff actuators in terms of stability robustness and fidelity of force control. Several impedance control architectures have been proposed for SEA. Among these alternatives, the cascaded controller with an inner-most velocity loop, an intermediate torque loop and an outer-most impedance loop is particularly favoured for its simplicity, robustness, and performance. In this paper, we derive the \emph{necessary and sufficient conditions} to ensure the passivity of this cascade-controller architecture for rendering two most common virtual impedance models. Based on the newly established passivity conditions, we provide non-conservative design guidelines to haptically display a null impedance and a pure spring while ensuring the passivity of interaction. We also demonstrate the importance of including physical damping in the actuator model during derivation of passivity conditions, when integral controllers are utilized. In particular, we show the adversary effect of physical damping on system passivity.
	\end{abstract}
	
	
	\section{Introduction}
	
	Ensuring natural and safe physical human-robot interactions (pHRI) is an active research area, since such interactions  form the basis of  successful applications in many areas, including service, surgical, assistive, and rehabilitation robotics. Safety of interaction requires the impedance characteristics of the robot at the interaction port to be controlled precisely~\cite{ColgatePhD}. Along these lines, many robot designs and several impedance control~\cite{Hogan1985a} schemes have been proposed.
	
	Many successful applications rely on open-loop force/impedance control to avoid the use of force sensors. In these approaches, the motor torques/impedances are directly mapped to the end-effector forces/impedance. The performance of open-loop control approaches relies on the transparency of the mechanical design. In particular, the mechanical design of the robot needs to have high stiffness, low inertia, and high passive backdrivability to ensure good performance by minimizing parasitic forces. Optimization techniques exist to help design robots with high transparency~\cite{hayward,ramazan}. However, the design of highly transparent robots become quite challenging, even infeasible, as high force/impedance levels are necessitated, since backdrivable high torque/power density actuators are not available.
	
	Many robotic systems rely on closed-loop force control to compensate for parasitic forces originating from  the mechanical design. However, the performance of closed-loop force controllers suffers from an inherent limitation imposed by the non-collocation of sensors and actuators.
	In particular, given that a force sensor needs to be attached to the interaction port,  there always exists inevitable compliance between the actuators and the force sensor. This non-collocation results in a fundamental performance limitation for the controller, by  introducing an upper bound on the loop gain of the closed-loop force-controlled system. Above this limit, the closed-loop system becomes unstable~\cite{Chae1987,Eppinger1987}.
	
	When traditional force sensors with  high stiffness are employed in the control loop, the stable loop gain of the system is mostly allocated for the force sensing element, and this significantly limits the upper bound available for the controller gains to achieve fast response and good robustness properties from the controlled system. Consequently, such force control architectures typically rely on high quality actuators/power transmission elements to avoid hard-to-model parasitic effects, such as friction and torque ripple, since these parasitic forces may not be effectively compensated by robust controllers based on aggressive force-feedback controller gains.
	
	Series elastic actuation (SEA) trades-off force-control bandwidth for force/impedance rendering fidelity, by introducing highly compliant force sensing elements into the closed-loop force control architecture~\cite{HowardPhD,Pratt1995}. By decreasing the force sensor stiffness, it allows higher force controller gains to be utilized for responsive and robust force-controllers. SEA can effectively mask the inertia of the actuator side from the interaction port, featuring  favorable output impedance characteristics that is safe for human interaction over the entire frequency spectrum. In particular, by modulating its output impedance to a desired level, SEA can ensure \emph{active} backdrivability, within the force control bandwidth of the device, through closed-loop impedance control of high power density actuators. For the frequencies over its control bandwidth, the apparent impedance of the system is limited by the inherent compliance of the force sensing element that acts as a physical filter against impacts, impulsive loads, and high frequency disturbances~\cite{HowardPhD,Pratt1995,Robinson1999,Sensinger2006b}.
	
	SEA is also preferred, since the cost of SEA robotic devices can be made significantly (about an order of magnitude) lower than traditional force sensor based implementations, as successfully demonstrated by the commercial Baxter robot~\cite{Baxter}. In particular, since the orders of magnitude more compliant force sensing elements in SEA experience significantly larger deflections with respect to commercial force sensors, regular position sensors, such as optical encoders, can be employed to measure these deflections, enabling the implementation of low-cost digital force sensing elements that do not require signal conditioning. Furthermore, since the robustness properties of the force controllers enable SEA to effectively compensate for parasitic forces, lower cost components can be utilized as actuators/power transmission elements in the implementation of SEA.  To date, a large number of SEA designs have been developed for a wide range of applications~\cite{Pratt1995,Robinson1999,Sensinger2006,Veneman2006,Khatib2008,Kyoungchul2012,Sarac2014,Erdogan2016,Otaran2016,Munawar2016,Woo2017,Caliskan2018,Senturk2018,Gillespie2014}.
	
	The main disadvantage of SEA is the significantly decreased closed-loop bandwidth caused by the increase of the sensor compliance~\cite{Pratt1995}. The determination of appropriate stiffness of the compliant element is an important aspect of SEA designs, where a compromise solution needs to be reached  between force control fidelity and closed-loop control bandwidth. In particular, higher compliance can increase the force sensing resolution, while higher stiffness can improve the control bandwidth of the system. Possible oscillations of the end-effector, especially when SEA is not in contact, and the potential energy storage of the elastic element may pose as other challenges of SEA designs, depending on the application. 	
	
	SEA is a multi-domain concept whose performance synergistically depends on the design of both the plant and the controller~\cite{Kamadan2017,Kamadan2018}. The high performance controller design for SEA to be used in pHRI has further challenges, since ensuring safety of interactions is an imperative design requirement that dominates the design process. In particular, the safety of interaction requires coupled stability of the controlled SEA together with the human operator. However, the presence of a human operator in the control loop significantly complicates the stability analysis, since a comprehensive model for human dynamics is not available. Particularly, human dynamics is nonlinear, time and configuration-dependent. Contact interactions with the environment pose similar challenges, since the impedance of the contact environment is, in general, nonlinear and uncertain.
	
	The coupled stability analysis of the pHRI system in the absence of human and environment model is commonly conducted using the frequency domain passivity framework~\cite{Colgate1988,Colgate1997}. This approach assumes that human operator behaves as a passive network element in the closed-loop system and does not intentionally generate energy to destabilize the system~\cite{Hogan1989}. Similarly, non-animated environments are passive. Therefore, coupled stability of the overall system can be concluded, if the closed-loop SEA with its controller can be guaranteed to be passive~\cite{Colgate_human}. Passivity framework is advantageous as it provides robust stability for a large range of human and environment models.  However, non-passive systems are not always unstable~\cite{Buerger2001} and the passivity is a relatively conservative condition that  imposes strict constraints on the controller gains to degrade the system performance.
	
	It is well-established that ensuring passivity adversely affects the transparency~\cite{Lawrence}, and this trade-off brings a challenge in the design of high-performance controllers that can ensure coupled stability. The trade-off between stability and transparency~\cite{hulin,zaad2,griffiths}, as well as the factors affecting the transparency have been investigated  in the literature~\cite{hirche,Peer,tavakoli}. While keeping coupled stability intact, a controller allowing better compromise between transparency and robust stability is desirable~\cite{Peer}.
	
	In this paper, we consider the cascaded control architecture for SEA, with an inner-most velocity loop, an intermediate torque loop and an outer-most impedance loop, and derive \emph{necessary and sufficient conditions} to ensure the passivity of this controller for haptic rendering of a null impedance and a pure spring. Our results rigorously extend the earlier reported sufficiency conditions on passivity of SEA and provide the least conservative range for passivity renderable impedances. We also demonstrate the importance of including physical damping in the actuator model during derivation of passivity conditions, when integral controllers are utilized.
	
	\begin{figure*}[b]
		\center
		\resizebox{1\textwidth}{!}{\includegraphics{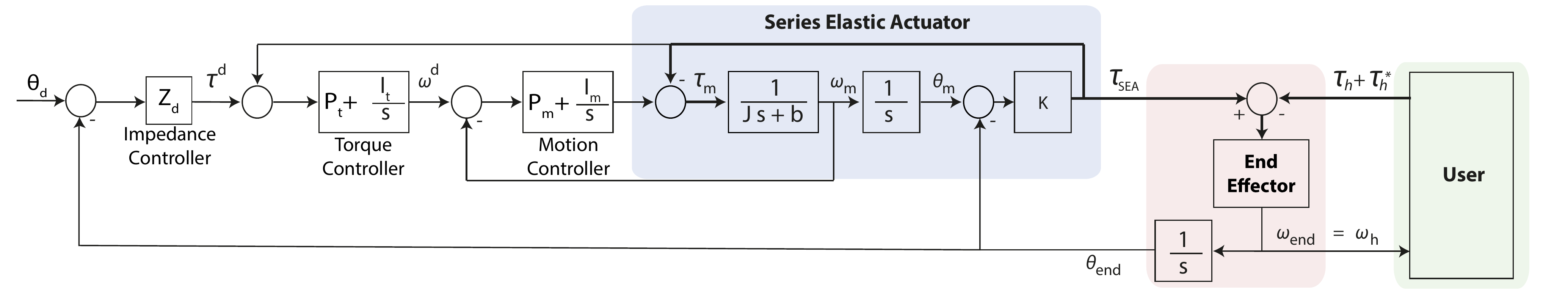}}
		\vspace{-1.5\baselineskip}
		\caption{Cascaded control of a series elastic actuator}
		\label{SEA_diagram}
	\end{figure*}
	
	The rest of the paper is organized as follows. Section~II reviews the related work and emphasizes the contributions of this paper in comparison to the related works. Section~III provides the system model considered in this study and lists the underlying assumptions together with their justification. Section~IV derives the necessary and sufficient conditions for passivity while rendering null impedance and pure spring. Rendering performance with respect to controller gains is studied systematically and design guidelines are provided in Section~V. Section~VI provides a discussion of the results by elaborating on the important differences with respect to earlier works. Finally, Section~VII concludes the paper.
	
	\section{Related Work}

The notion of intentionally introduced compliance between the actuator and the end effector for force controlled robotic joints has been first proposed in~\cite{HowardPhD}. 
Later, the term ``series elastic actuator"~(SEA) was coined for this force control approach~\cite{Pratt1995}. The SEA controller in~\cite{Pratt1995} is based on a single force-control loop, where the actuator is torque controlled based on the deflection feedback from the compliant element.  Similarly, a PID controller with feed-forward acceleration terms to compensate the actuator inertia has been proposed in~\cite{Robinson1999}. These early strategies rely on low-pass filters instead of pure integrators to preserve passivity, at the expense of allowing steady state errors under constant disturbances.
	
	Other control techniques for SEA include disturbance observer based force controllers~\cite{Kyoungchul2012,Paine2014} and controllers based on high order derivatives~\cite{Abe2012}. While linear models are most widely adapted for force/torque control of SEAs, there also exist some nonlinear control strategies~\cite{nonlinearSEA,nonlinearSEA2,nonlinearSEA3,nonlinearSEA4}.
	
	A fundamentally different architecture based on cascaded control loops has first been proposed  in~\cite{HowardPhD} and later rediscovered in~\cite{Pratt2004,Wyeth2008}. In this approach, an inner-loop controls the velocity of the actuator, rendering the system into an ``ideal" motion source, while an outer-loop controls the interaction force based on the deflection feedback from the compliant element. Wyeth called this approach \emph{velocity-sourced SEA}~\cite{Wyeth2008}, emphasizing that most of the earlier work considered the motor as a torque source rather than a velocity source. This particular strategy allows for the use of integrators; thus the closed-loop controlled system can effectively counteract constant disturbances at the steady state.  This architecture also allows for utilization of well-established robust motion controllers for the inner-loop to counteract parasitic effects of friction and stiction. Furthermore, the controller can be tuned easily without the need for precise actuator dynamics.  The cascaded control approach has been widely utilized in various applications~\cite{Veneman2006,Tokatli2010,Sarac2014,Erdogan2016,Otaran2016,Munawar2016,Caliskan2018,Senturk2018}.
	
	Using the cascaded control architecture, Vallery \textit{et~al.} derived and experimentally verified \emph{sufficient conditions} to ensure passivity of the impedance rendering, for the case of zero reference torque~\cite{Vallery2007}. They have suggested simple yet quite conservative guidelines: select a proportional velocity gain that is greater than the motor inertia, and select integrator gains that are less than the half of the corresponding proportional gains. In their later work, Vallery \textit{et al.} conducted a theoretical analysis and an experimental study for pure spring rendering~\cite{Vallery2008}. In this work, it has been proven that, for the cascaded control architecture, the passively renderable virtual stiffness is bounded by the stiffness of physical spring employed in the SEA.
	
	For a variety of viscoelastic virtual impedance models, Tagliamonte \textit{et~al.} performed a theoretical analysis using the cascaded control architecture, but also including the damping coefficient in the actuator dynamics~\cite{Tagliamonte2014}. In this work, they have proposed less conservative \emph{sufficient conditions} to ensure passivity with properly selected controller gains, for the cases of null impedance and pure spring rendering. They have also demonstrated that the Voigt model, that is, linear spring and damping elements in parallel connection, cannot be passively rendered using the cascaded control architecture.
	
	Recently, Fiorini \textit{et~al.} surveyed  different impedance and admittance control architectures for SEA and summarized \emph{sufficient conditions} for passive impedance rendering with basic impedance control, velocity-sourced impedance control, collocated admittance control and collocated impedance control architectures \cite{calanca2017}. This study concludes that similar bounds on passively renderable impedances exist for all four control architectures and these limits can be extended, if ideal acceleration feedback can be used to predict and cancel out the influence of load dynamics. Noise and bandwidth restrictions of acceleration signals and potential overestimation of feed-forward signals resulting in feedback inversion are important practical challenges that have limited the adaptation of the acceleration-based control approach since initially proposed in~\cite{Pratt1995,Robinson1999}.
	
	This paper builds upon earlier works on passivity of velocity-sourced impedance control of series elastic actuators~\cite{Vallery2007,Vallery2008,Tagliamonte2014} and extends their results by providing the \emph{necessary and sufficient conditions} to ensure passive rendering of null impedance and pure springs. Our results not only provide rigourous sufficiency proofs, but also relax the earlier established bounds by extending the range of impedances that can be passively rendered via cascaded control architecture. Based on the newly established necessary and sufficient conditions, design guidelines are provided to select controller gains to reach optimal performance while maintaining passivity.
	
	Furthermore, our results prove the necessity of a second bound on the integral gains due to existence of physical damping in the system. This bound has been overlooked in the literature~\cite{Vallery2007,Vallery2008}, as it is counter-intuitive for additional dissipation to result in more strict conditions on controller gains. However, this bound is crucial in practice, as it is imposed due to inevitable physical damping of the actuator; hence, cannot be safely neglected, if integral controllers are used in both inner motion and intermediate torque control loops. We also remark that the damping term counterintuitively reduces the Z-width of the system, that is, the dynamic range of passively renderable impedances, as also reported in~\cite{Tagliamonte2014}.

	\section{System Description}
	
	Figure~\ref{SEA_diagram} depicts the block diagram of velocity-sourced impedance control for SEA. In particular, the cascaded controller is implemented with an inner motion control loop to render the system into an ideal motion source, and the outer force/torque control loop generates references for the motion control loop such that the spring deflection is kept at the desired level to match the reference force. The interaction torque is measured through the linear spring troque that is proportional to the difference between motor position $\theta_{m}$ and end-effector position $\theta_{end}$.
	
	The dynamics of the SEA model consist of actuator inertia~$J$, viscous friction~$b$, and the linear spring constant~$K$. PI controllers are employed for both velocity and torque control loops. At the outermost loop, an impedance controller is employed to generate references to the torque controller depending on the desired impedance $Z^{d}$ to be displayed around the equilibrium position $\theta^{d}$.
	
	Some simplifying assumptions are considered while developing the SEA model and its control architecture, as in~\cite{Tagliamonte2014}. These assumptions include:
	\begin{itemize}
		\item To develop a linear time-invariant (LTI) model, nonlinear effects like stiction, backlash and motor saturation are neglected. In the literature, it has been demonstrated that the cascaded force-velocity control scheme can effectively overcome the problems of stiction and backlash~\cite{Wyeth2008,Sensinger2006}. If the motor is operated within its linear range, then the other nonlinear effects, like motor saturation, also vanish.
		\item The overall inertia and damping of the SEA are considered to be on the motor side. The inertia of the load is not included in the analysis, since the load inertia does not contribute to the passivity conditions.
		\item Electrical dynamics of the system is approximated based on the commonly employed assumption that electrical time constant of the system is orders of magnitude faster than the mechanical time constant.
		\item It is assumed that motor velocity signal is available with a negligible delay. This assumption is realistic for electronically commuted motors furnished with Hall effect sensors. Furthermore, for motors furnished with high-resolution encoders, differentiation filters running at high sampling frequencies (commonly on hardware) can be employed to result in real-time estimation of velocity signals with very small delay, within the bandwidth of interest.
		\item Without loss of generality, for simplicity of analysis, zero reference trajectory is assumed for the equilibrium position (i.e. $\theta^{d}=0$) and transmission ratios are set to unity.
	\end{itemize}
	
	Conventionally, the output impedance $Z_{out}$ of the closed loop system is defined at its output port as the relationship between the conjugate variables $\omega_{end}$(s) and $\tau_{SEA}$(s) as
	\begin{equation}
	Z_{out} = - \frac{\tau_{SEA}(s)}{\omega_{end}(s)}= - \frac{\tau_{SEA}(s)}{s\theta_{end}(s)}
	\end{equation}
	
	The following analysis is performed based on Eqn.~(1) as it defines the relationship at the interaction port of the human/environment and the end-effector of SEA.
	
	\section{Passivity Analysis}
	
	An LTI and single-input single-output (SISO) system, whose transfer function is denoted as H(s) is passive if and only if the following conditions hold~\cite{ColgatePhD}:
	\begin{itemize}
		\item[(i)] $H(s)$ must have all its poles in the open left half plane.
		\item[(ii)] $Re\{H(jw)\} \geq 0$ for all $w  \in  (-\infty, \infty)$ for which $jw$ is not a pole of $H(s)$.
		\item[(iii)] Poles on the imaginary axis are allowed only if they are simple and have positive real residues.
	\end{itemize}

	Condition~(i) implies (isolated) stability, but all three conditions are required to be simultaneously satisfied for passivity.
	
	Condition~(ii) justifies the neglect of the load side inertia of the SEA plant, since it only adds an imaginary term to the frequency response. The necessary and sufficient conditions for the passivity of the system depicted in Figure~1 for positive system parameters and control gains are derived by using this theorem as follows.

	\subsection{Null Impedance Rendering}
	
	Let us first analyse the case of null impedance (i.e $Z_{d}$=0), which also corresponds to the special case where the outer-most impedance loop is cancelled and zero set-point reference signal is fed to the torque controller (i.e $\tau^{d}=0$). This particular case is interesting as it is commonly employed to ensure the active backdrivability of SEA. From Eqn.~(1), the output impedance is expressed as
	\begin{equation}
	Z_{out}^{null}=\frac{K \: s\: (J \: s^2+(P_{m}+b) \: s+I_{m})}{D_{Z}(s)}
	\label{null_impedance}
	\end{equation}
	\noindent where
	\begin{equation}
	D_{Z}(s)=Js^4 + (P_{m}+b) s^3 + (K + \gamma) s^2+ \alpha K s + K I_{m}I_{t}
	\label{char_eq}
	\end{equation}
	\noindent with $\alpha = P_{m} \: I_{t} + P_{t}\: I_{m}$ and
	$\gamma=K\:P_{m}\:P_{t} + I_{m}$.
	
	Let us determine the controller gains that guarantee passivity.  Naturally, the parameters $J$ and $b$ that capture the motor dynamics and the spring constant $K$ are always positive. It is established in classical control theory that if any one of the coefficients of the characteristic equation is non-positive in the presence of at least one positive coefficient, then the system is unstable~\cite{ogata}. Along these lines, we also assume that all controller gains are selected as positive. This selection satisfies the necessary condition for the stability of the system.
	
	The method of Hurwitz determinants or Routh's stability criterion can be used to assess the stability of a system, which is the first condition for it to be passive. The Routh array of a fourth order system with a characteristic equation of the form $a_{0}s^4+a_{1}s^3+a_{2}s^3+a_{3}s+a_{4}$  reads as
	
	\begin{table}[H]
		\centering
		$\begin{array}{ccc}
		a_0 & \hspace{5mm} a_2 & \hspace{5mm} a_4 \\
		a_1 & \hspace{5mm} a_3 & \hspace{5mm} 0 \\
		(a_1 a_2 - a_0 a_3)/a_1 & \hspace{5mm} a_4 & \hspace{5mm} 0 \\
		(a_1 a_2 a_3 - a_1^2 a_4 - a_0 a_3^2) / (a_1 a_2 - a_0 a_3) & \hspace{5mm} 0 & \hspace{5mm} 0  \\
		a_4 & \hspace{5mm} 0 &  \hspace{5mm} 0
		\end{array}$
		\vspace{-.5\baselineskip}
		
	\end{table}
	\noindent It follows from the Routh array that the following two inequalities need to be satisfied to ensure stability.
	\begin{eqnarray}
	a_{1}a_{2}-a_{0}a_{3}>0 \label{eqn_r1} \\
	a_{1}a_{2}a_{3}-a_{0}a_{3}^2-a_{4}a_{1}^2>0 \label{eqn_r2}
	\end{eqnarray}
	
	Note that if Eqn.~(\ref{eqn_r2}) is satisfied, then Eqn.~(\ref{eqn_r1}) is also met, as can be proven by multiplying Eqn.~(\ref{eqn_r1}) with $a_{3}$, and noting that Eqn.~(\ref{eqn_r2}) ensures that $a_{1}\:a_{2}\:a_{3}-a_{0} \: a_{3}^2 > a_{4} \: a_{1}^2 >\ 0$. Hence, if we define a variable as $\xi:=a_{1}a_{2}a_{3}-a_{0}a_{3}^2-a_{4}a_{1}^2$ the system is stable if and only if $\xi>0$. The value of $\xi$ in terms of our system parameters reads as
	{\small
		\begin{equation}
		\xi=\alpha K \: (P_{m}+b)(K +\gamma)- K\: I_{m} \: I_{t} \: (P_{m}+b)^2 -J\: K^2 \: \alpha^2 \label{eqn:stability}
		\end{equation}}
	The inequality $\xi>0$ represents Condition~(i) for passivity. As for Condition~(ii), we have to assess the positive-realness of $Z_{out}^{null}(jw)$. It is relatively involved to examine the positive-realness of a complex fraction directly. Along these lines, we use a polynomial that provides us with the same information about the sign of the real part of $Z_{out}^{null}(jw)$. This polynomial approach has been used in several earlier works, including~\cite{Dylan2017}. For completeness of the presentation, below we repeat this well-established result from the literature.
	
	\begin{prop} {\rm
			For ease of notation, denote the frequency response of a SISO LTI system as $H(jw)= {\rm num}(jw)/{\rm den}(jw)$. Then, ${\rm sign}(Re\{ H(jw) \})= {\rm sign}(P(w))$ for any value of $w$ for which ${\rm den}(jw) \neq 0$, where ${\rm sign}(\cdot)$ represents the signum function and $P(w)$ is a polynomial defined as $P(w)=Re\{{\rm num}(jw) {\rm den}(-jw)\}=\sum_{i}^{} d_{i} \: w^i$.}
	\end{prop}
	
	\begin{proof}
		Multiply numerator and denominator of $H(jw)$ with the complex conjugate of the denominator as	
		\begin{equation}
		H(jw)=\frac{{\rm num}(jw)}{{\rm den}(jw)} = \frac{{\rm num}(jw) {\rm den}(-jw)}{{\rm den}(jw) {\rm den}(-jw)}
		\nonumber
		\end{equation}
		Since the  denominator of the resulting fraction is never negative and is zero only when ${\rm den}(jw)$ is zero, we conclude that the proposition holds.
	\end{proof}
	
	Consequently, $P(w)\geq0$ is a necessary and sufficient test to ensure Condition~(ii) for passivity.
	\noindent
	For the system described by Eqn.~(\ref{null_impedance}), $P(w)$ evaluates to
	\begin{equation}
	P(w)=d_{2}w^2+d_{4}w^4
	\label{null_passivity}
	\end{equation}
	\noindent where the coefficients are defined as
	\begin{eqnarray}
	d_{2} &=& K^2(P_{t}I_{m}^2-bI_{t}I_{m}) \label{eqn:d2}\\
	d_{4} &=& K^2(P_{m}+b+P_{t}P_{m}^2+bP_{t}P_{m}-J\alpha) \label{eqn:d4}
	\end{eqnarray}
	It will be proven that $d_{2}\geq0 \wedge d_{4}\geq0$ is a necessary and sufficient condition to ensure $P(w)\geq0$ for $\forall w \in \mathds{R}$
	
	\begin{proof}[\emph{Proof.} Sufficiency]
		Since there are only even powers of $w$ in $P(w)$, the image of $P(w)$ is non-negative if all coefficients are also non-negative.
	\end{proof}
	
	\begin{proof}[\emph{Proof.} Necessity]
		Rearrange Eqn.~(\ref{null_passivity}) as $P(w)=w^2 \: (d_{2}+d_{4} \: w^2)$. Then, $P(w)\geq0$ for $w \in (-\infty, \infty)$ if and only if $d_{2}+d_{4} \: w^2\geq0$. The roots to this simple quadratic expression (i.e., $d_{2}+d_{4} \: w^2$) are equal to $\pm \sqrt{-d_{2}/d_{4}}$.
		
		If the signs of $d_{2}$ and $d_{4}$ agree, there is no real root to this expression, meaning that its graph never crosses the horizontal axis. Thus, if $d_{2}$ and $d_{4}$ are positive, then the graph has to lie above the abscissa for all $w$ values. On the other hand, if the coefficients have opposite sign, there will be two real roots forcing the parabola to cross the abscissa and go below zero. In this case, $P(w)$ is negative for $w \in (-\sqrt {-d_{2}/d_{4}} , \sqrt {-d_{2}/d_{4}})$.
		
		Finally, in the extreme case where either coefficient is zero the other coefficient must be greater than or equal to zero for $P(w)$ to be non-negative. 
	\end{proof}
	
	Thus, $P(w)\geq0 \iff d_{2}\geq0 \wedge d_{4}\geq0$. Consequently, the \emph{necessary and sufficient conditions} for the passivity of the system whose closed-loop impedance is characterized by Eqn.~(\ref{null_impedance}) can be expressed as follows:
	\vspace{-1mm}
	{\small	\begin{eqnarray}
		\xi=K[\alpha (P_{m}+b)(K+\gamma)- I_{m}I_{t}(P_{m}+b)^2 -J K\alpha^2]> 0	\label {ineqn:stability} \\
		d_{2}=K^2\:[I_{m}(P_{t}I_{m}-bI_{t})] \geq 0 \label{ineqn:d2}\\
		d_{4}=K^2\:[(P_{m}+b)(1+P_{m}\:P_{t})-J\alpha)] \geq 0 \label{ineqn:d4}
		\end{eqnarray}}	
	\begin{prop} {\rm
			The necessary and sufficient conditions to passively render zero impedance (or equivalently zero force/torque) for the cascaded controlled SEA shown in Figure~\ref{SEA_diagram} with positive control gains are as follows:
			
			{\small \begin{gather}
				\left[ J< \frac{(P_{m}+b)(1+P_{m}\:P_{t})} {P_{m}\:I_{t}+P_{t}\:I_{m}} \wedge b\leq \frac{P_{t}\:I_{m}}{I_{t}} \right] \label{cond:null1} \\
				\lor \nonumber \\
				\left[ J\leq \frac{(P_{m}+b)(1+P_{m}\:P_{t})} {P_{m}\:I_{t}+P_{t}\:I_{m}} \wedge b< \frac{P_{t}\:I_{m}}{I_{t}}\right] \label{cond:null2}
				\end{gather} }}
	\end{prop} \label{prop:cond:null}

	In the sequel step-by-step proof is provided.
	
	\smallskip
	\begin{lemma} {\rm
			$(d_{2}\geq0 \wedge d_{4}>0) \lor (d_{2}>0 \wedge d_{4}\geq0) \implies \xi>0$ }
	\end{lemma}
	
	This statement implies that the Inequality~(\ref{ineqn:stability}) does not add extra restriction to the system of inequalities composed of Eqns.~(\ref{ineqn:stability}),~(\ref{ineqn:d2}) and~(\ref{ineqn:d4}), except that $d_{2}$ and $d_{4}$ cannot be zero simultaneously. In other words, $d_{2}$ and $d_{4}$ are non-negative, but only one of them can be zero at a time. Otherwise, the system is unstable.
	
	The lemma contains two statements that are connected with the logical \textit{or} operator. To facilitate understanding the discussion, the proof will be subdivided into these two parts.
	
	\begin{proof}[\emph{Proof.} Part I: $\boldsymbol{(d_{2}\geq 0 \wedge d_{4}>0)}$]
		Inequality~(\ref{ineqn:d2}) dictates an upper bound on $b$. According to Eqn.~(\ref{ineqn:d2}), the maximum value for the motor damping without violating passivity with the given controller gains can be computed as
		\begin{equation}
		b\leq\frac{P_{t} \: I_{m}}{I_{t}}=b_{max}
		\label{bound_damping}
		\end{equation}
		
		Inequality~(\ref{ineqn:d4}) dictates an upper bound on $J$. According to Eqn.~(\ref{ineqn:d4}), the maximum value for the motor inertia without violating passivity with the given controller gains can be computed as
		\begin{equation}
		J\leq\frac{(P_{m}+b)(1+P_{m}\:P_{t})}{P_{m} \: I_{t} + P_{t}\:I_{m}}=J^{null}_{max}
		\label{bound_inertia}
		\end{equation}
		
		Now, assume  the control gains are selected so that the motor inertia is less than its maximum allowable value. In other words, $J=J^{null}_{max}-\epsilon$ where $0<\epsilon<J^{null}_{max}$. This selection entails $d_{4}>0$. After substituting this value of $J$ in Eqn. (\ref{eqn:stability}), $\xi$ becomes
		\begin{equation}
		\xi=\epsilon \:K^2\: \alpha^2+K\:I_{m}\:(P_{m}+b)\:(P_{t}\:I_{m}-b\:I_{t})
		\end{equation}
		
		Here, only the last term can make the expression negative, but this is avoided when Eqn.~(\ref{bound_damping}) is met. Thus, we conclude that $d_{2}\geq0 \wedge d_{4}>0 \implies \xi>0$.
	\end{proof}
	
	\begin{proof}[\emph{Proof.} Part II: $\boldsymbol{(d_{2}>0 \wedge d_{4}\geq 0)}$]
		Assume the control gains are selected so that the motor inertia takes its maximum allowable value that is, $J=J^{null}_{max}$. Substituting this value of $J$ into Eqn.~(\ref{eqn:stability}) yields the following expression.		
		\begin{equation}
		\xi=K\:I_{m}\:(P_{m}+b)\:(P_{t}\:I_{m}-b\:I_{t})
		\label{eqn:passive_stability}
		\end{equation}
		
		Clearly, $\xi$ is positive if $J\leq J^{null}_{max}$ and $b<b_{max}$. Thus, passivity is ensured when $d_{2}>0$ and $d_{4}\geq0$. However, when $J=J^{null}_{max}$ and $b=b_{max}$ the value of $\xi$ evaluates to zero, which implies instability. Thereby, the system is not stable when $d_{2}=0$ and $d_{4}=0$. 
	\end{proof}

	Consequently, $(d_{2}\geq0 \wedge d_{4}>0) \lor (d_{2}>0 \wedge d_{4}\geq0)$ constitutes the most general solution set that solves Eqns.~(\ref{ineqn:stability}), (\ref{ineqn:d2}) and~(\ref{ineqn:d4}) concurrently, unless negative system parameters or controller gains are allowed. This concludes the proof of Proposition~2.

	\subsection{Pure Spring Rendering}
	
	In this section, we analyze the case where a virtual spring of stiffness $K_{d}$ is displayed. When $Z_{d}$ is set to $K_{d}$, the output impedance $Z_{out}^{spr}$ reads as
	
	{\small \begin{equation}
		Z_{out}^{spr}=	K\: \frac{Js^{4}+(P_{m}+b)s^{3}+\delta s^{2}+\alpha K_{d}s+K_{d} I_{m}I_{t}}{sD_{Z}(s)}\label{pure_spring}
		\end{equation}}%
	\noindent where $\delta =P_{m}\:P_{t}\:K_{d}+I_{m}$. The remaining intermediate parameters are the same as in the case of null impedance. Only a single pole located at the origin is added to the characteristic equation in Eqn.~(\ref{char_eq}). Note that, this does not cause a violation of Condition~(iii), since the pole on the imaginary axis is simple and have a positive residue as shown below.
	
	\begin{equation*}
	\res_{s=0} Z_{out}^{spr}=\lim_{s \to 0} s \: Z_{out}^{spr}=\frac{K_{d}^{2}}{K}>0
	\end{equation*}
	\noindent Therefore, Eqn.~(\ref{ineqn:stability}) for stability must also be adopted here.

	The nonzero coefficients of $P(w)$ for this system are as follows:
	\begin{eqnarray}
	d_{4} &=& K[(K-K_{d})\beta-\alpha K K_{d}] \label{eqn:d4s}\\
	d_{6} &=& K[(K-K_{d})\eta+K(P_{m}+b)]  \label{eqn:d6}
	\end{eqnarray}
	\noindent	where $\beta=P_{t}I_{m}^2-bI_{m}I_{t}$ and	$\eta=P_{m}^2 P_{t}+P_{m} P_{t} b - J\alpha$ .
	
	\smallskip
	
	Note that, $P(w)\geq0 \iff d_{4}\geq0 \wedge d_{6}\geq0$ as can be proven by rearranging $P(w)$ as $w^4(d_{4}+d_{6}w^2)$ and following the same reasoning as in the previous case. Thus, the necessary and sufficient conditions for the passivity of system whose closed-loop impedance is characterized by Eqn.~(\ref{pure_spring}) are as follows
	{\small \begin{eqnarray}
		\xi=K[\alpha (P_{m}+b)(K+\gamma)- I_{m}I_{t}(P_{m}+b)^2 -J K\alpha^2]>0 \nonumber \\
		d_{4}=K[(K-K_{d})\beta-\alpha KK_{d}]\geq0 \label{ineqn:d4s}\\
		d_{6}=K[(K-K_{d})\eta+K(P_{m}+b)]\geq0 \label{ineqn:d6}
		\end{eqnarray}}
	Eqns.~(\ref{ineqn:d4s}) and (\ref{ineqn:d6}) stipulate some bounds on the renderable virtual stiffness.
	From Eqn.~(\ref{ineqn:d4s}), we get the following upper bound for the renderable stiffness if $\beta + \alpha K$ is positive.
	\begin{equation}
	K_{d}\leq K\frac{\beta}{\beta+\alpha K}<K
	\label{spring_bound}
	\end{equation}
	Inequality~(\ref{spring_bound}) puts an upper bound on the physical damping. If $\beta$ is negative, but $\beta + \alpha K$ is positive, then~ Eqn.~(\ref{spring_bound}) states that one cannot render a spring of any stiffness, since the maximum value for $K_d$ would be a negative number. To ensure that $\beta>0$, we need to employ the same bound on damping found in Eqn.~(\ref{bound_damping}). However, particular attention must be paid when $\beta + \alpha K$ is negative (in which case $\beta$ is automatically negative). In this case, the controlled system becomes unstable as will be shown later. For the time being, we continue the analysis with the assumption of positive $\beta$ (and hence positive $\beta+ \alpha K$).\par
	From Eqn.~(\ref{ineqn:d6}), we get the following upper bound for the renderable stiffness.
	\begin{equation}
	K_{d} \leq K\frac{\eta+P_{m}+b}{\eta}
	\label{spring_bound_d6}
	\end{equation}
	
	Clearly, the value of $K_{d}$ that satisfies Eqn.~(\ref{spring_bound}) also satisfies the less constraining inequality in Eqn.~(\ref{spring_bound_d6}). Inequality in Eqn.~(\ref{spring_bound}) shows that if passivity is desired under the cascaded control architecture, the rendered stiffness must be strictly less than the stiffness of the physical spring employed in the SEA plant, which was originally reported in~\cite{Vallery2008} excluding the damping term. Thus, the maximum value of the desired stiffness can be set to

	{\small \begin{eqnarray}	
		K_{d}^{max}&=&K\frac{\beta}{\beta+\alpha K} \nonumber\\
		&=& K\frac{P_{t}I_{m}^2-bI_{m}I_{t}}{P_{t}I_{m}^2-bI_{m}I_{t}+K( P_{m} I_{t} + P_{t}I_{m})}
		\label{kd_max}
		\end{eqnarray}}
	\begin{prop}{\rm
			The necessary and sufficient conditions to passively render a virtual spring for the system in Fig.~\ref{SEA_diagram} with positive control gains are
			{\small \begin{gather}
				\!\!\!\!\!\! \left[ J \!\!<\!\! \frac{(P_{m}+b)(\Delta K\:P_{m}\:P_{t}+K)}{\alpha \Delta K} \label{cond:spring1}
				\wedge
				b < \frac{P_{t}\:I_{m}}{I_{t}}
				\wedge
				K_{d}\!\!\leq K_{d}^{max}  \right] \\
				\lor \nonumber \\
				\!\!\!\!\!\! \left[ J \!\!\leq \!\! \frac{(P_{m}+b)(\Delta K\:P_{m}\:P_{t}+K)}{\alpha \Delta K}
				\wedge
				b < \frac{P_{t}\:I_{m}}{I_{t}}
				\wedge
				K_{d}\!\!< K_{d}^{max} \right] \label{cond:spring2}
				\end{gather}}
			\noindent	where $\Delta K:=K-K_{d}$ and $K_{d}^{max}$ is as in Eqn.~(\ref{kd_max}).}
	\end{prop} \label{[prop:cond:spring}
	
	\begin{proof}
		From Eqn.~(\ref{eqn:d6}),
		\begin{equation}
		d_{6}=\Delta K(P_{m}+b)(\Delta K\:P_{m}\:P_{t}+K)-\Delta K\:J\: \alpha \geq0
		\label{eqn:d6_d4}
		\end{equation}
		
		Eqn.~(\ref{eqn:d6_d4}) introduces an upper bound on the motor inertia~$J$ as
		\begin{equation}
		J\leq \frac{(P_{m}+b)(\Delta K\:P_{m}\:P_{t}+K)}{\alpha \Delta K}=J_{max}^{spr}
		\label{bound_inertia_spring}
		\end{equation}
		
		Note that $J_{max}^{spr}$ is not only a function of control gains, but also a function of the desired stiffness $K_{d}$ to be rendered. If we set $K_{d}$ to its maximum allowable value given in Eqn.~(\ref{kd_max}), $J_{max}^{spr}$ reads as
		\begin{equation}
		J_{max}^{spr} = \frac{(P_{m}+b)(P_{t}I_{m}^2+\alpha K(1+P_{m}P_{t})-bI_{m}I_{t})}{\alpha^2 K}
		\label{bound_inertia2}
		\end{equation}
		Substituting Eqn.~(\ref{bound_inertia2}) into Eqn.~(\ref{eqn:stability}) yields $\xi=0$, which implies instability. Hence, when $d_{4}$ and $d_{6}$ are simultaneously zero, the system is not stable. Following the similar arguments as in the null impedance case, it can be proven that $(d_{4}\geq0 \wedge d_{6}>0) \lor (d_{4}>0 \wedge d_{6}\geq0) \implies \xi>0$; hence, the conditions in Eqn.~(\ref{cond:spring1}) or  Eqn.~(\ref{cond:spring2}) hold.
	\end{proof}
	
	Now let us analyse the case when $\beta + \alpha K<0$ for completeness. In this case, Eqn.~(\ref{spring_bound}) modifies to
	
	\begin{equation}
	K_{d}\geq K\frac{\beta}{\beta+\alpha K}>0
	\label{spring_bound2}
	\end{equation}
	
	Here,  Eqn.~(\ref{spring_bound2}) introduces a lower bound on the renderable stiffness. In other words, following inequalities must satisfied to ensure $d_4\geq0 \wedge d_6\geq0$.
	
	\begin{equation}
	K\frac{\beta}{\beta+\alpha K} \leq K_d \leq K\frac{\eta+P_{m}+b}{\eta}
	\label{falan}
	\end{equation}
	
	However, considering Eqns.~(\ref{bound_inertia_spring}) and~(\ref{eqn:stability}),  $K_d$ values in this range will result in $\xi \leq 0$ which implies instability.
	
	\begin{remark}{\rm
			\begin{itemize}
				\item[]
				\item[-] While deriving the passivity conditions, positive controller gains are considered, since negative gains are hardly used in practice and make the analysis much harder to follow.
				
				\item[-] It should be pointed out that the integral gains $I_{m}$ and $I_{t}$ may assume zero values. A naive interpretation of Proposition~2 might lead to a misconclusion that passivity is lost when no velocity integral gain is employed (i.e., $I_{m}=0$), since there will always be some damping $b$ present in the plant. However, since these conditions are derived for positive control gains, the analysis needs to be extended to include zero gains. In particular, each integrator increases the degree of the system by one.  In the case of null impedance, when no integrators are employed (i.e., $I_{m}=I_{t}=0)$, the output impedance is a second order system that is unconditionally passive.
				
				\item[-] Table I reports the necessary and sufficient conditions for ensuring passivity when null impedance is rendered with only one integral gain. Note that, the direct dependence on $b$ for passivity vanishes in these cases.
				
				\begin{table}[h] 
					\begin{center}
						\caption{The necessary and sufficient conditions for passivity when one integrator gain is set to zero}
						\begin{tabular}{l|l}        				
							$I_{m}=0$ & $J<J_{max}^{null}|_{I_{m}=0}$ \\
							\hline
							$I_{t}=0$ & $J\leq J_{max}^{null}|_{I_{t}=0}$ \\
						\end{tabular}
					\end{center} 
				\end{table}
				
				\item[-] In the case of a pure spring, when $I_{m}=0$, the system cannot passively render a virtual spring of any stiffness. This is surprising in that usually integrators are known to jeopardize passivity, but in this case, a minimum amount of integral gain is necessary to render an impedance passively. When only $I_{t}=0$, Proposition~3 remains valid.
				
				\item[-] Note that null impedance is mathematically equivalent to zero virtual stiffness. Consequently, if $K_{d}$ is set to zero, Proposition~3 reduces to Proposition~2.
		\end{itemize}}
	\end{remark}

	\section{Analysis of Rendering Performance}
	
	In this section, the effect of the controller gains on the system response is analyzed through Bode plots. Visualization of passivity through Bode plots is convenient, as passivity implies that the phase is restricted to the interval $[-90^\circ, 90^\circ]$. If the phase goes anywhere beyond this interval at any frequency, then the system is not passive.
	
	Since PI controllers are employed for both the inner velocity and the outer torque control loops, there are four controller parameters to choose namely, $P_{m},P_{t}, I_{m}$, and $I_{t}$. Firstly, Bode plots are drawn with respect to the changes in a certain controller gain, (e.g., $P_{m}$ or $I_{t}$) while keeping the other three gains constant, so as to analyze the effect of each individual parameter on the behaviour of the system. Next, design guidelines are outlined to choose the controller gains that render the system passive, while exhibiting good performance for haptic impedance rendering. Realistic values for the SEA plant parameters used in all simulations are reported in Table II.
	\begin{table}[h!]
		\begin{center}
			\caption{Physical parameters considered for the SEA plant}
			\begin{tabular}{l|l}
				\hline
				\multicolumn{2}{c}{Mechanical Parameters of SEA} \\
				\hline
				J & 0.2  kg m$^2$ \\
				b & 3 Nm s/rad \\
				K & 250 Nm/rad  \\
				\hline
			\end{tabular} \vspace{-1.5\baselineskip}
		\end{center}
	\end{table} \label{Tab:sea_parameters}

	\subsection{Effects of Controller Gains on Null Impedance Rendering} \label{Sec:null_imp}
	
	In this section, we analyze the effect of each controller gain in the case of null impedance rendering. For each simulation, we start with a base case scenario with certain controller gains reported in Table~III. Then, we increase each gain individually to see its effect on the system response through Bode plots.
	\begin{table}[h!]
		\begin{center}
			\caption{Nominal controller gains to render null impedance}
			\begin{tabular}{l|l}
				\hline
				\multicolumn{2}{c}{Controller Gains} \\ 
				\hline
				$P_{m}$ & 20 Nm s/rad \\
				$P_{t}$ & 5  rad/(s Nm) \\
				$I_{m}$ & 10 Nm/rad  \\
				$I_{t}$ & 5  rad/(s$^2$ Nm)  \\
				\hline
			\end{tabular} \vspace{-1.5\baselineskip}
		\end{center}
	\end{table}\label{Tab:base_contr_gains}
	
	\begin{figure}[b]
		\resizebox{\columnwidth}{!}{\includegraphics{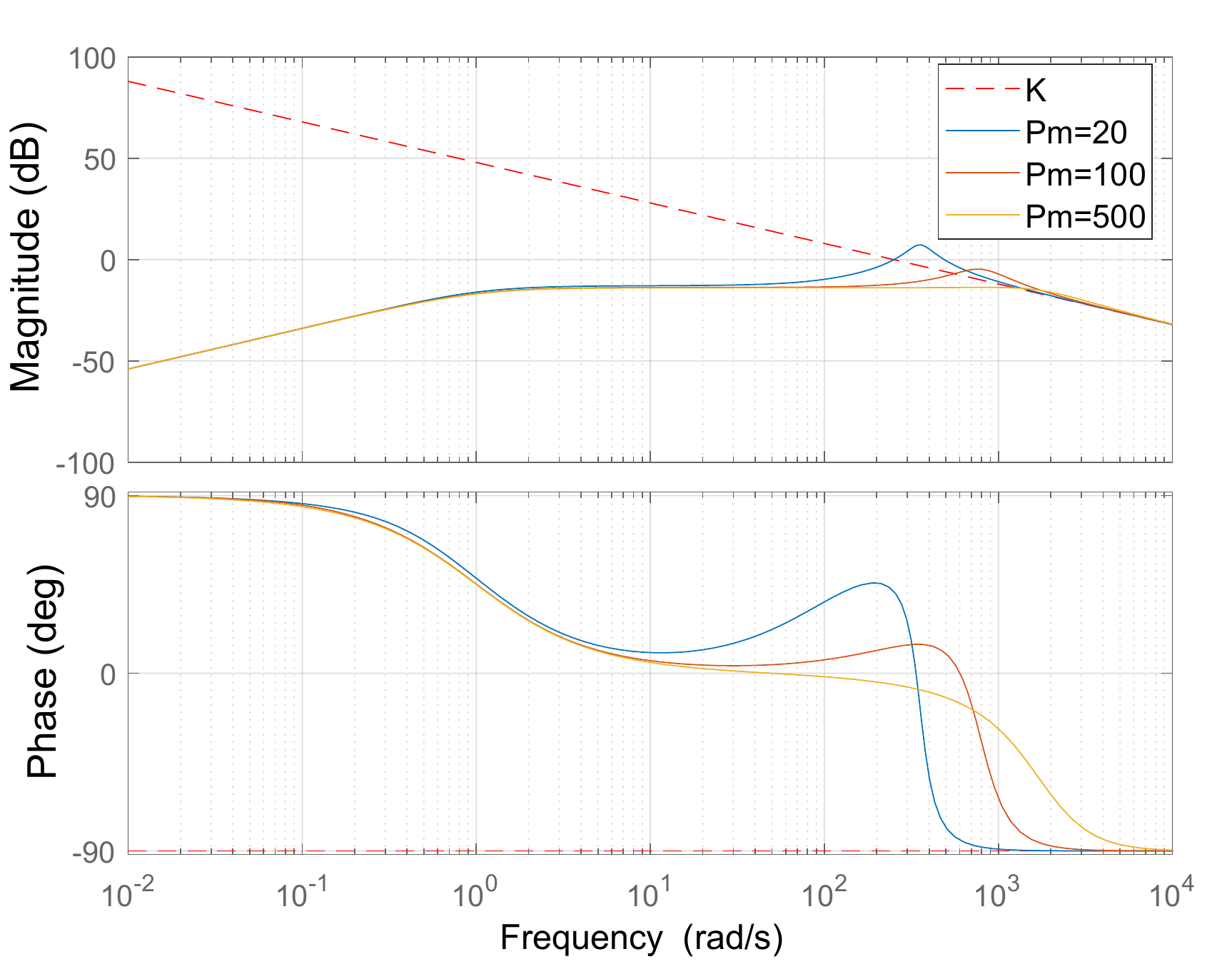}}
		\vspace{-1.25\baselineskip}
		\caption{Null impedance rendering with various velocity proportional gains $P_m$}
		\label{null_Pm}
	\end{figure}
	
	\begin{figure}[t]
		\vspace{-.75\baselineskip}
		\resizebox{\columnwidth}{!}{\includegraphics{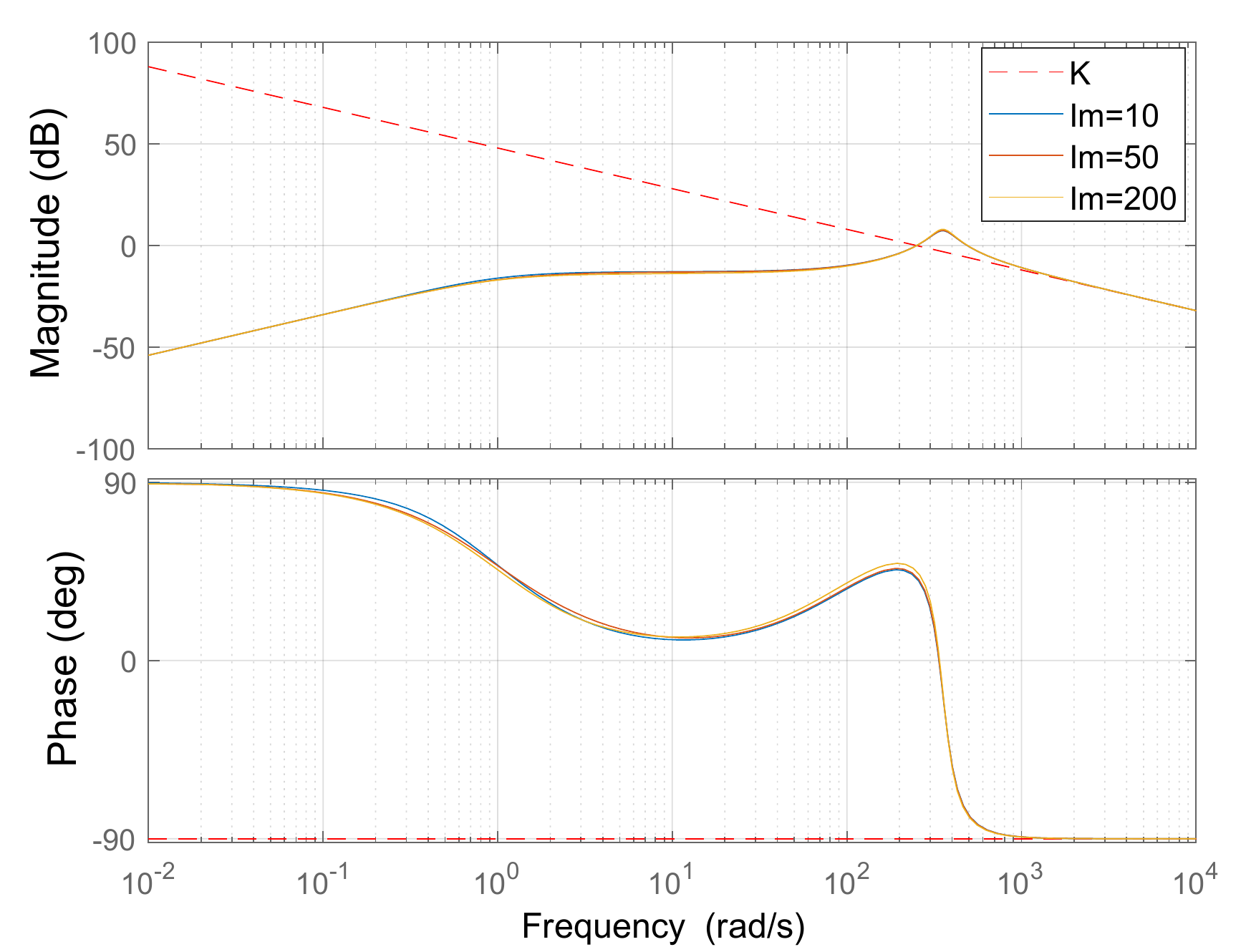}}
		\vspace{-1.5\baselineskip}
		\caption{Null impedance rendering with various velocity integral gains $I_m$}
		\vspace{-1.\baselineskip}
		\label{null_Im}
	\end{figure}
	
	It is observed that the system behaviour may be grouped into three phases. In the first phase, where the input frequency has a low value, the system displays the characteristics of a pure inertia. In the second phase, where the input frequency has an intermediate value, viscous damping behaviour is observed. In the third phase, where the input frequency has a high value, the system response reduces to that of the physical spring employed in the SEA plant. As argued earlier, this is due to the fact that the compliance between the actuator and the load acts as a physical filter against high-frequency force components, which provides safety and robustness against unexpected collisions and impacts.
	
	Figure~\ref{null_Pm} shows the effect of the velocity proportional gain $P_m$. Plots are constructed with different controller gains, and the legend indicates the gain values used for the analysis.
	
	Plots indicate that $P_{m}$ has no significant effect in the first phase (i.e., the inertial zone), but it helps to smooth out the transition from the second phase to the third phase by decreasing the resonant peak at the corresponding cut-off frequency. Theoretically, there exists no bound on $P_{m}$ that causes violation of passivity. However, a practical bound is likely to be imposed by physical bandwidth limitation of the actuator.

	Figure~\ref{null_Im} shows the effect of the velocity integral gain $I_{m}$. Plots indicate that $I_{m}$ has a negligible effect on overall system response. On the other hand, an increase in $I_{m}$ helps to preserve passivity against the actuator damping bound (i.e., $b_{max}$), while too much increase may jeopardize passivity against the actuator inertia bound (i.e., $J_{max}$), as can be seen from Proposition~2.

	\begin{figure}[b]
		\resizebox{\columnwidth}{!}{\includegraphics{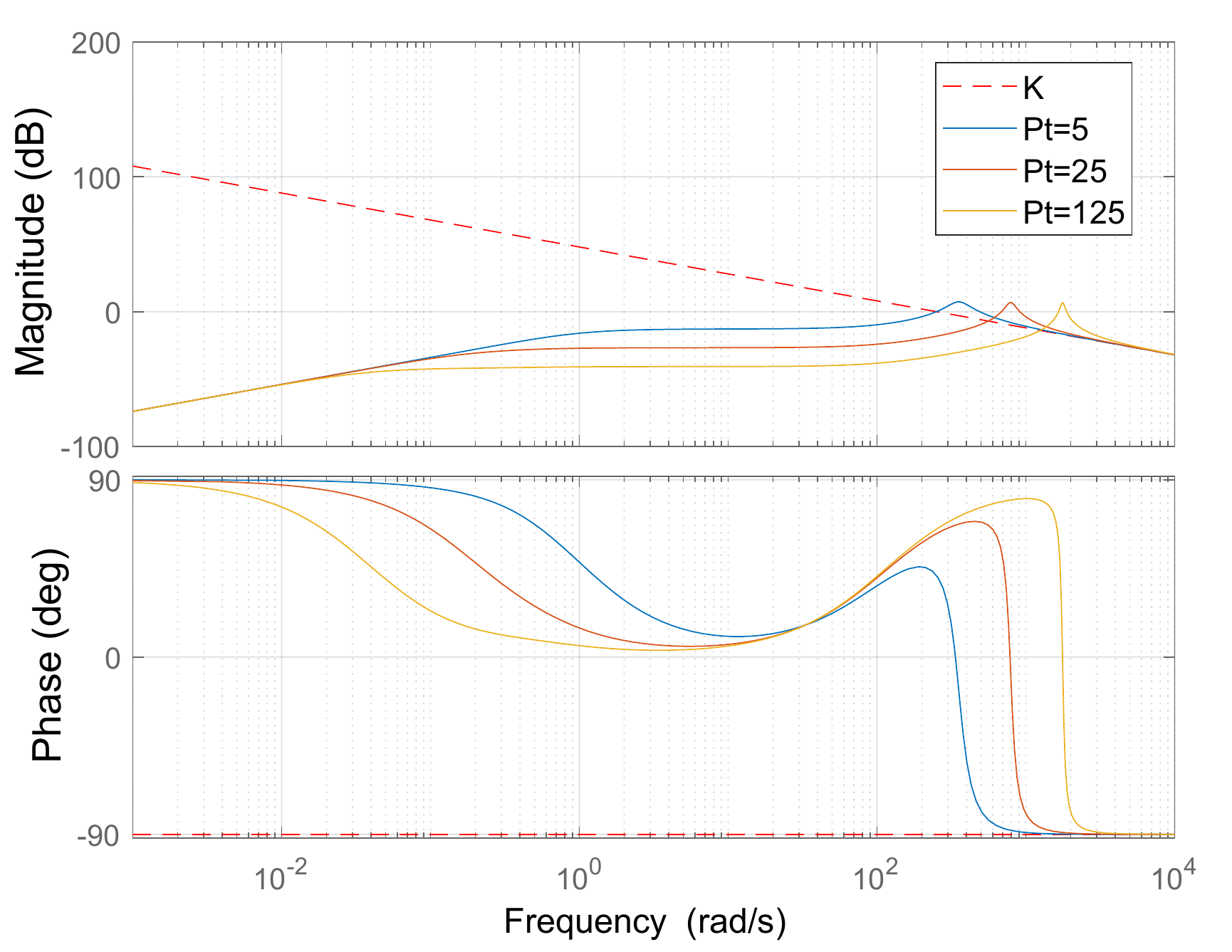}}
		\vspace{-1.5\baselineskip}		
		\caption{Null impedance rendering with various torque proportional gains $P_t$}
		\vspace{-1\baselineskip}		
		\label{null_Pt}
	\end{figure}
	
	\begin{figure}[t]
		\vspace{-1.\baselineskip}
		\resizebox{1\columnwidth}{!}{\includegraphics{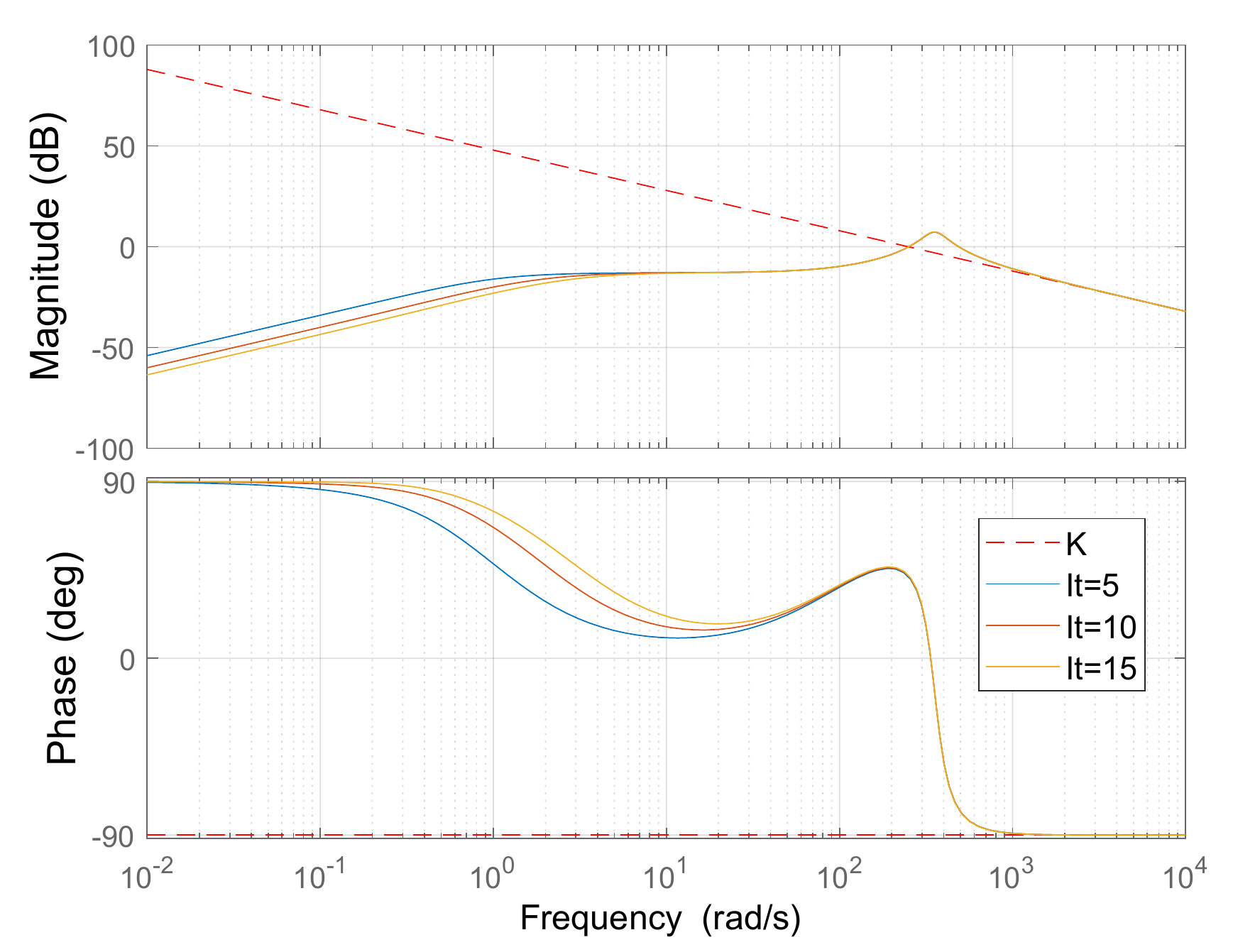}}
		\vspace{-1.5\baselineskip}		
		\caption{Null impedance rendering with various torque integral gains $I_t$}
		\vspace{-1.\baselineskip}
		\label{null_It}
	\end{figure}
	
	Figure~\ref{null_Pt} shows the effect of the torque proportional gain $P_{t}$. Plots indicate that larger values of $P_{t}$ shrink the inertial zone, which may not be favorable. On the other hand, since the system reaches the damping zone earlier, the apparent impedance stays lower for larger $P_t$, as can be inspected from the magnitude plots. Hence, the selection of $P_{t}$ involves a trade-off between the control bandwidth and transparency performance. If the operating frequency of the application is low, then $P_{t}$ may be chosen high.
	
	Figure~\ref{null_It} shows the effect of the torque integral gain $I_t$. Plots indicate that an increase in $I_{t}$ dramatically improves system performance, since not only the inertial zone gets enlarged, but also the apparent inertia is lowered. However, there exists an upper bound on $I_{t}$ due to the passivity conditions given in Proposition~2.
	
	\subsection{Design Guidelines for Null Impedance Rendering}
	
	The analysis shows that the outer torque controller is the main determinant of performance during null impedance rendering. Increasing $I_{t}$ results in better rendering performance by reducing the apparent inertia, as well as widening the inertial zone. While the analysis indicates that the inner velocity controller does not have a significant effect on system response, in practice a fast and robust controller is still desirable to render the actuator as an ideal motion generator under unmodelled parasitic forces.
	
	As $I_{t}$ gets larger, which is desired for better performance, passivity is at stake as can be seen from Eqn.~(\ref{bound_damping}). Hence, relatively aggressive gain values for $P_{m}$ and $I_{m}$ are recommended to design a robust inner motion controller, as well as to preserve passivity without sacrificing good null impedance rendering performance.

	\subsection{Effects of Controller Gains on Pure Stiffness Rendering}
	
	In this section, we analyze the effect of each controller gain while rendering a virtual spring, using a similar approach as in Section~\ref{Sec:null_imp}.
	
	Once again, it is observed that the overall behaviour of the system may be grouped into three phases. In the first phase, the virtual stiffness of the desired value is displayed. In the second phase, damping behaviour is observed. In the third phase, as expected, the system behaviour reduces to that of the physical spring employed in the SEA. The numerical values for system parameters used in simulations are reported in Table~IV.
	
	\begin{table}[b]
		\begin{center}
			\caption{Nominal controller gains to render a pure spring}
			\begin{tabular}{l|l}
				\hline
				\multicolumn{2}{c}{System parameters} \\ 
				\hline
				$P_{m}$ & 20 Nm s/rad \\
				$P_{t}$ & 30 rad/(s Nm) \\
				$I_{m}$ & 100 Nm/rad  \\
				$I_{t}$ & 5  rad/(s$^2$ Nm)  \\
				$K_{d}$ & 50 Nm/rad\\
				\hline
			\end{tabular}
		\end{center}
	\end{table} \label{Tab:spring_params}

	Figure~\ref{spring_Pm} shows the effect of the velocity proportional gain~$P_{m}$.  Plots indicate that $P_{m}$ does not have a significant effect in the first phase, but high values of $P_{m}$ lower the resonant peaks that occur at the phase transitions. Theoretically, there exists no bound on $P_{m}$ that causes passivity violation.
	
	\begin{figure}[t]
		\resizebox{\columnwidth}{!}{\includegraphics{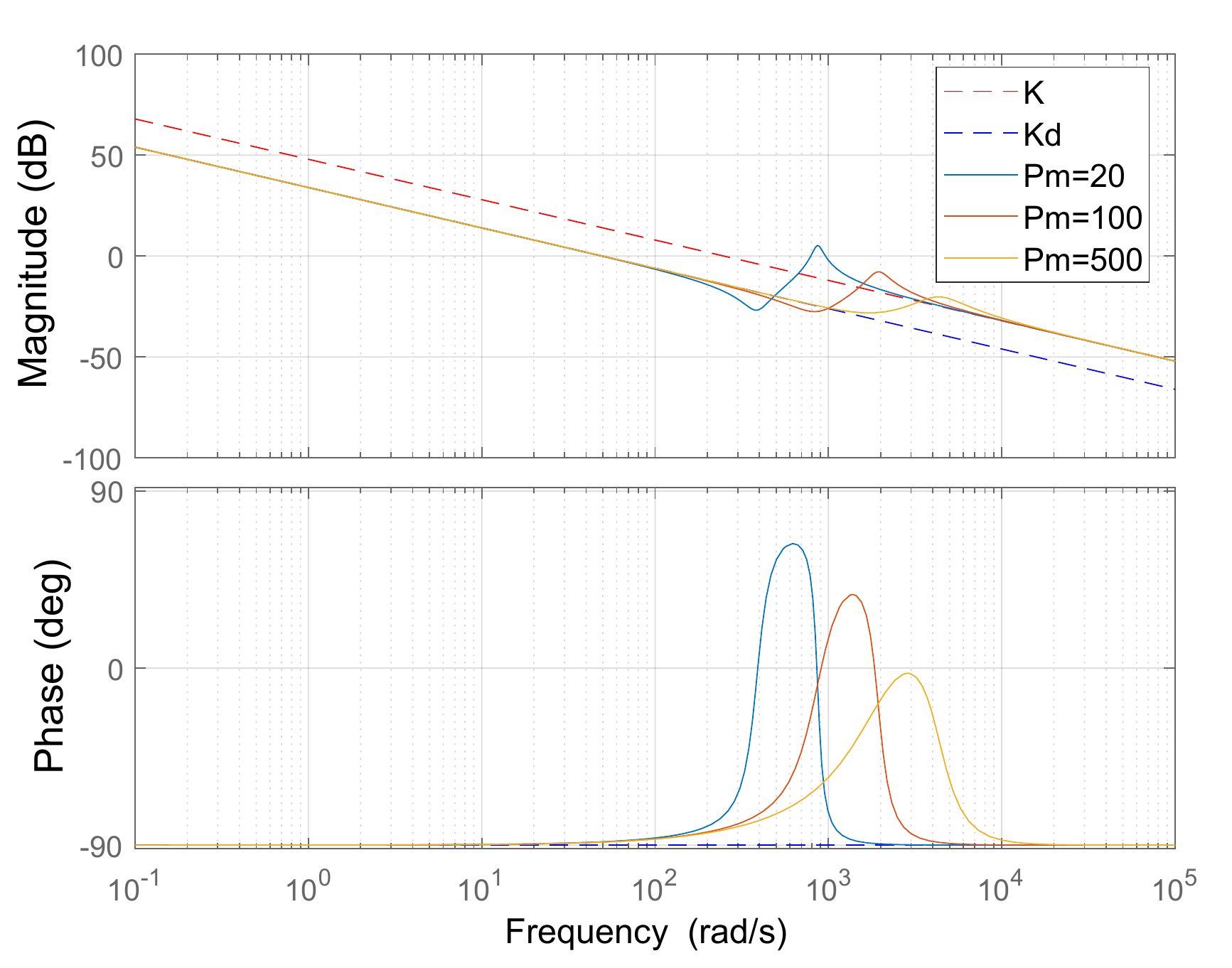}}
		\vspace{-1.5\baselineskip}		
		\caption{Pure spring rendering with various velocity proportional gains~$P_m$}
		\vspace{-1\baselineskip}
		\label{spring_Pm}
	\end{figure}

	\begin{figure}[b]
		\vspace{-1.\baselineskip}	
		\resizebox{\columnwidth}{!}{\includegraphics{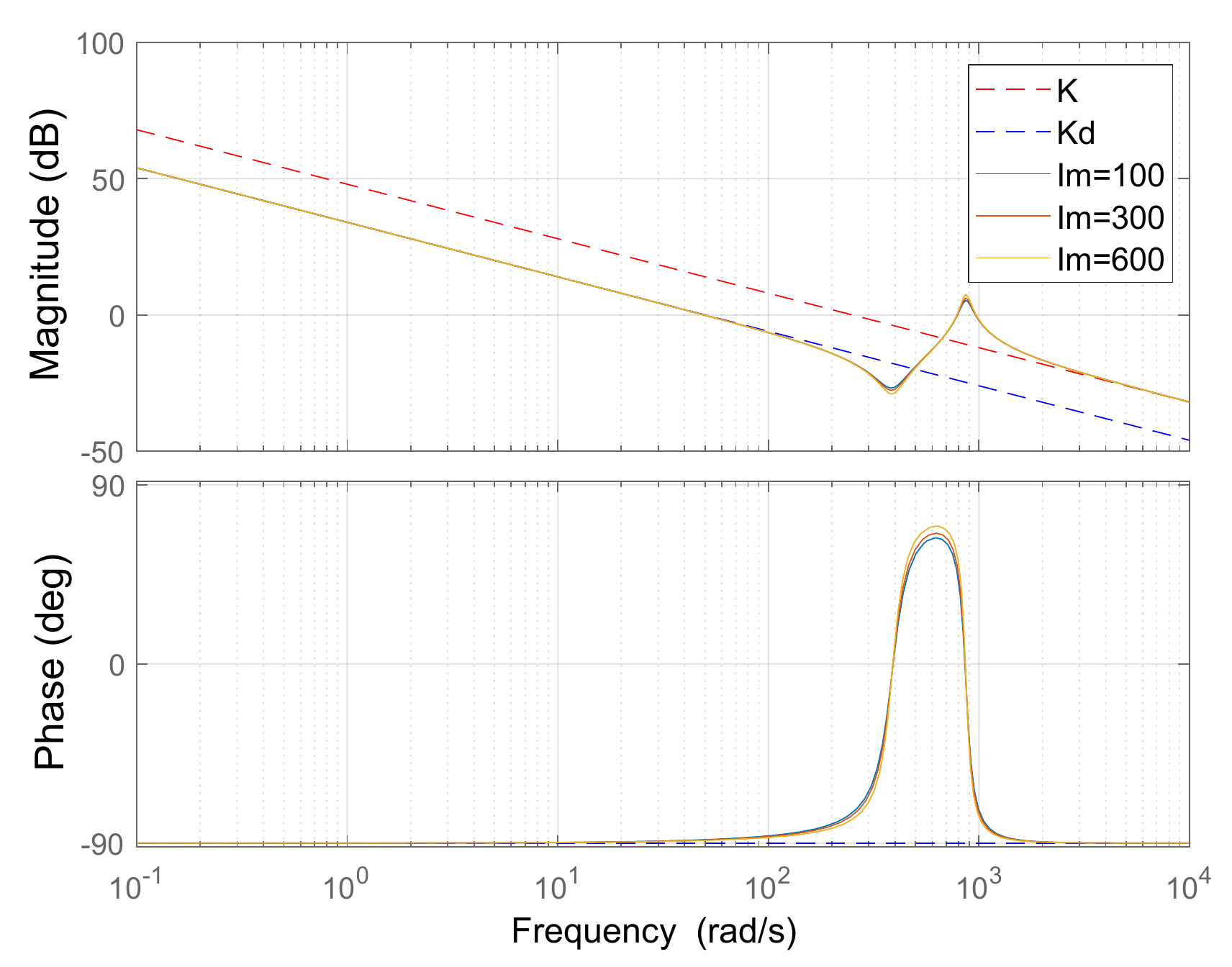}}
		\vspace{-1.5\baselineskip}		
		\caption{Pure spring rendering with various velocity integral gains~$I_m$}
		\label{spring_Im}
	\end{figure}
	
	Figure~\ref{spring_Im} shows the effect of the velocity integral gain $I_{m}$. Plots indicate that $I_{m}$ does not have a significant effect on the overall system response. However, it is the most critical parameter to determine the maximum renderable stiffness $K_{{d}_{max}}$. It can be seen from Eqn.~(\ref{kd_max}) that $K_{{d}_{max}} \to K$ as $I_{m} \to \infty$. For any other controller gain, this limit goes to a value less than the physical stiffness $K$ of SEA.
	
	\begin{figure}[t]
		\vspace{0.5\baselineskip}	
		\resizebox{\columnwidth}{!}{\includegraphics{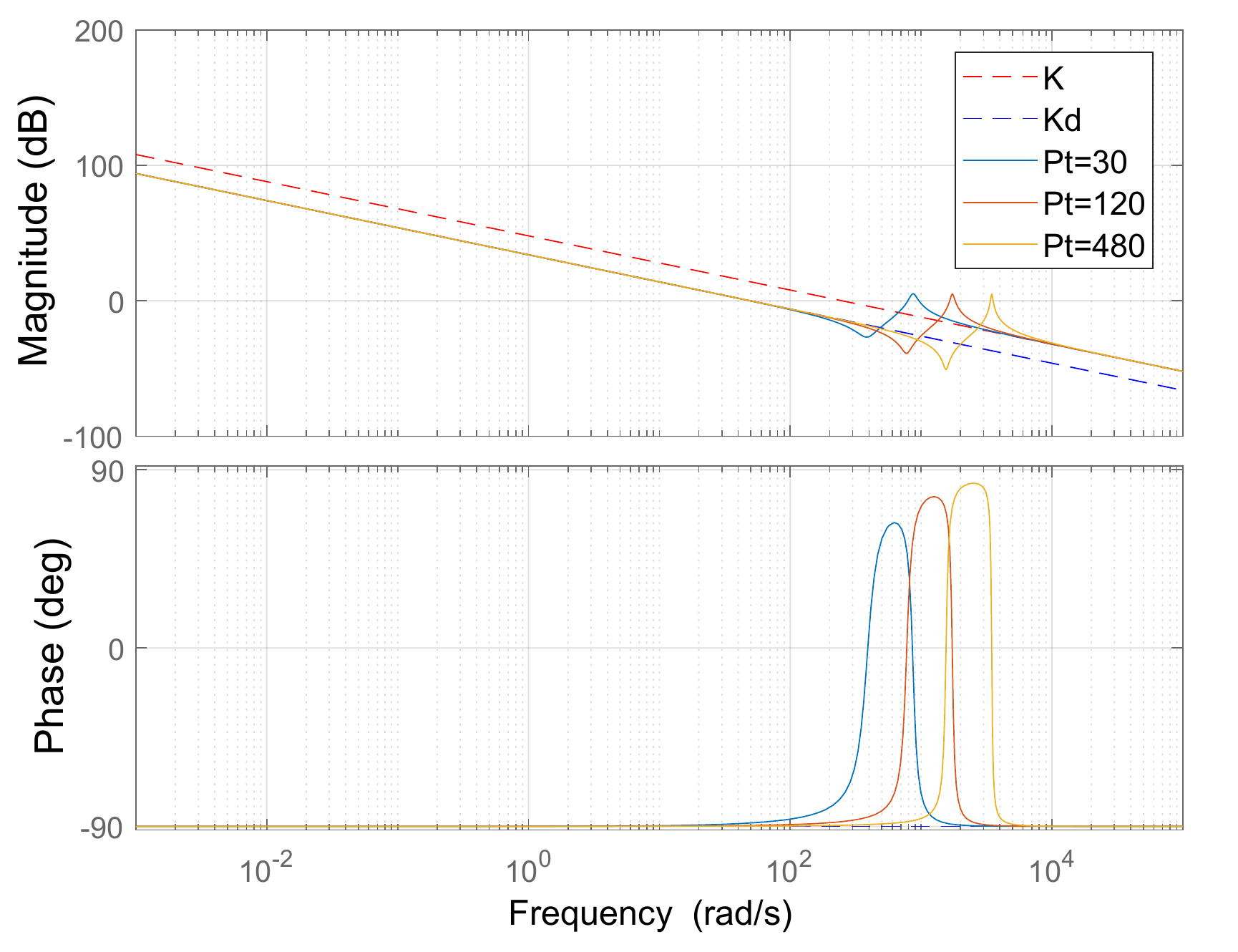}}
		\vspace{-1.5\baselineskip}		
		\caption{Pure spring rendering with various torque proportional gains~$P_t$}
		\vspace{-1.1\baselineskip}
		\label{spring_Pt}
	\end{figure}
	
	Figure~\ref{spring_Pt} shows the effect of the torque proportional gain $P_{t}$. Plots indicate that increasing $P_{t}$ provides better performance, since the desired stiffness is successfully rendered for a wider frequency range. However, on the downside, it also increases the resonant peaks at the phase transitions.
	
	\begin{figure}[b]
		\vspace{-1.\baselineskip}	
		\resizebox{\columnwidth}{!}{\includegraphics{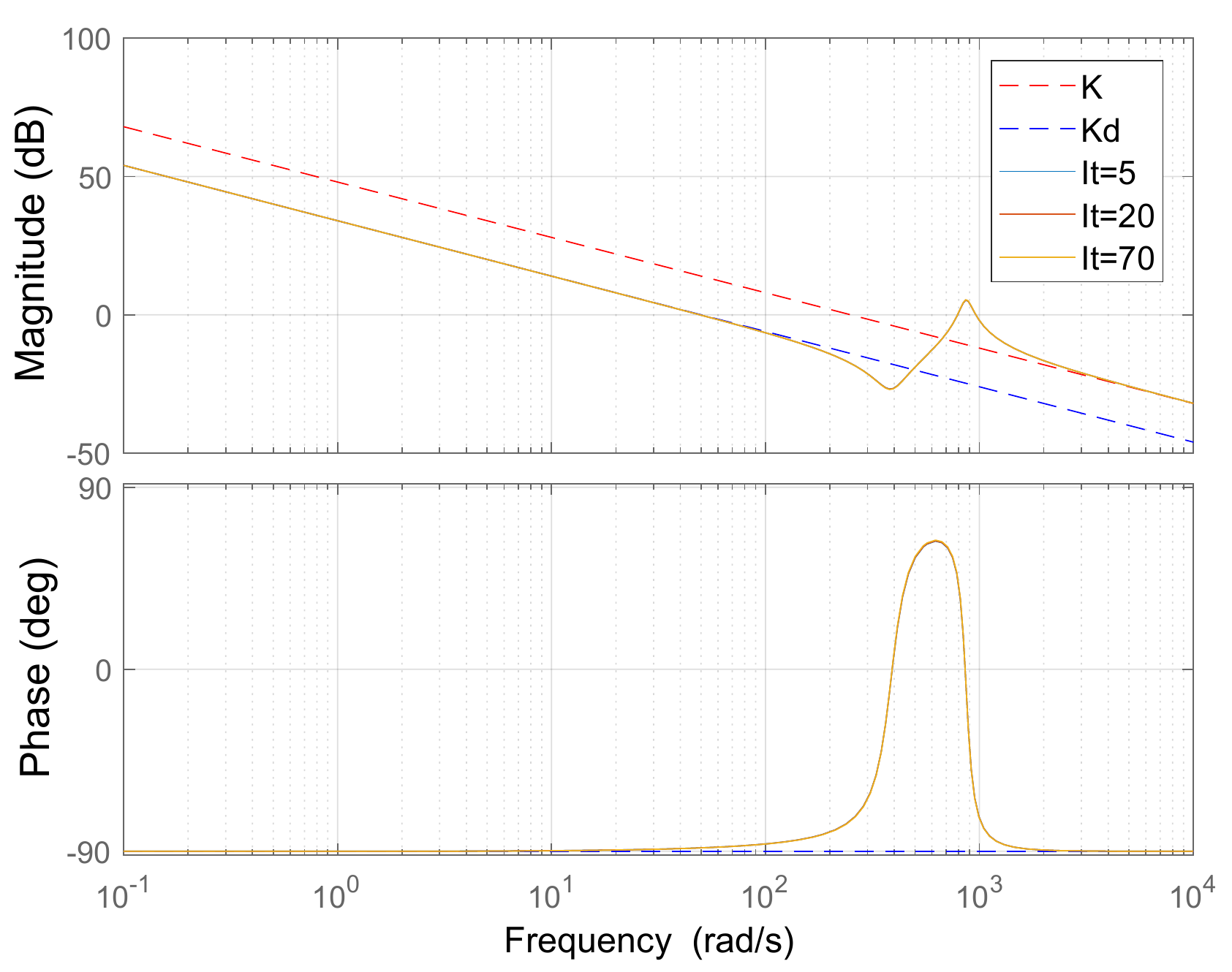}}
		\vspace{-1.5\baselineskip}		
		\caption{Pure spring rendering with various torque integral gains~$I_t$}
		\label{spring_It}
	\end{figure}
	
	Figure~\ref{spring_It} shows the effect of the torque integral gain $I_t$. Plots indicate that it does not significantly affect the system response. Moreover, large values of $I_{t}$ jeopardize passivity, as can be seen from Proposition~3. If $I_{t}$ is set to zero, the value of $I_{m}$ must be set to
	$I_{m}\geq KK_{d}/\Delta K$, in order to be able to display desired stiffness $K_{d}$, as can be seen from Eqn.~(\ref{spring_bound}). Along these lines, while it is theoretically alluring to set $I_{t}$ to zero, while a small $I_t$ may be preffered in practical implementations to eliminate steady-state errors.

	\subsection{Design Guidelines for Pure Stiffness Rendering}
	
	The analysis indicates that $P_{t}$ is the main determinant of the performance during pure stiffness rendering, since the frequency up to which the desired stiffness is successfully displayed is directly related to $P_{t}$. On the other hand, a sufficiently large value of $I_{m}$ must be employed to ensure that desired stiffness can be passively rendered. Hence, $I_{m}$ plays a crucial role in determining the K-width of the system.
	
	A high value of $P_{m}$ is also preferred as it smoothens the transitions between phases. Furthermore, an increase in $P_{m}$ helps to preserve passivity according to Proposition~3.
	
	$I_{t}$  does not have a significant effect on rendering performance. Moreover, it has an adverse effect on preserving passivity. A small value of $I_{t}$  may be injected into the system to eliminate the steady state errors due to constant disturbances, such as parasitic forces due to stiction.  P control may be applied at the outer torque controller, if the system does not suffer from undesirable steady-state response.
	
	\subsection{Overall Design Guidelines}
	
	Many applications in pHRI require frequent switching between active backdrivability (null impedance rendering) and virtual fixtures (pure spring rendering), such as haptic virtual environments that contain a unilateral constraints~\cite{Gillespie2017}. Hence, it is desirable to determine a single set of controller gains that ensures passivity for both impedance models. Luckily, the passivity bounds on controller gains while rendering for both impedances are not in conflict with each other. Hence, to design a passive controller that has the good  performance both impedances, the bounds provided in Proposition~2 and~3 can be used to adjust the controller gains (especially the integrator gains) sufficiently high to meet the the specifications of the application without jeopardizing passivity.

	\section{Discussion}
	
	Passivity is one of the most widely adopted paradigms to ensure coupled stability in pHRI applications. While this framework provides stability robustness, it imposes conservative conditions on the controller design that degrades the performance, as also demonstrated in the previous section. Therefore, it is of paramount importance to find the most relaxed passivity conditions for a system at hand. However, deriving non-conservative passivity conditions proves to be a challenging problem in most cases. As a consequence, \emph{sufficient} but not necessary passivity conditions are commonly employed, which impose further limitations in addition to the inherent conservativeness of the frequency domain passivity framework.
	
	In this paper, we have rigorously derived the necessary and sufficient conditions to passively render two widely adopted impedance models of zero impedance and pure stiffness under the prominent cascaded velocity-force control architecture. These results provide the least conservative bounds for all positive controller gains.
	
	\begin{table}[b]
		\begin{center}
			\caption{Design Guidelines for Rendering Null Impedance}
			\vspace{-.5\baselineskip}
			\resizebox{.9\columnwidth}{!}{\begin{tabular}{l|l}
					\hline
					Vallery~\emph{et al.}~\cite{Vallery2007,Vallery2008}  & $P_{m}>J \wedge P_{m}>2I_{m} \wedge P_{t}>2I_{t}$  \\
					\hline
					&
					\\    		
					Accoto~\emph{et al.}~\cite{Tagliamonte2014} & $J<$ {\large $ \frac{(P_{m}+b)(P_{m}\:P_{t})} {P_{m}\:I_{t}+P_{t}\:I_{m}}$} $\wedge$ {\large $b<\frac{P_{t}I_{m}}{I_{t}}$}
					\\
					&
					\\    			
					\hline
					&
					\\
					Ours & $J<$  {\large $\frac{(P_{m}+b)(1+\:P_{m}\:P_{t})}{P_{m}\:I_{t}+P_{t}\:I_{m}} $} $\wedge$ {\large $b<\frac{P_{t}I_{m}}{I_{t}}$}  \\ 	\\		
					\hline
			\end{tabular}}
		\end{center}
	\end{table}
	
	In particular, Tables~V and~VI report the passivity bounds for the system model in Figure~1 for rendering a null impedance and a virtual spring, respectively. The notations used in~\cite{Vallery2007,Vallery2008,Tagliamonte2014} are mapped to ours to enable easier comparisons. Note that results provided in~\cite{Vallery2007,Vallery2008,Tagliamonte2014} are  sufficient but not necessary conditions. In particular, the bounds reported in~\cite{Vallery2007,Vallery2008} are quite conservative. While the bounds provided in~\cite{Tagliamonte2014} relax the previously established passivity constraints~\cite{Vallery2007,Vallery2008}, these bounds still remain conservative.
	
	The difference between the conditions reported in~\cite{Tagliamonte2014} and our results are relatively small for null impedance rendering case, while it becomes more pronounced for pure spring rendering case. In particular, for null impedance rendering case, the bound on inertia is relaxed by a factor of $(1+1/(P_m I_t))$, while the bound on $b$ stays the same. However, the necessity of the bound on $b$ was proven for the first time in the present work. This allowed us to remark the unexpected adversary effect of physical damping on system passivity, as it unintuitively implies that too much dissipation may violate passivity.

	For the spring rendering case, the bound on $J$ is relaxed by a factor of $1+(K/(P_m P_t \Delta K))$, where $\Delta K \triangleq K - K_d$. Hence, the smaller $\Delta K$ (i.e., the stiffer virtual spring rendered), the less strict the bound on $J$ becomes. Finally, the bound on $K_d$ and $b$ remain the same as it has been reported in the literature \cite{Tagliamonte2014}. Note that the presence of damping not only imposes an additional passivity constraint but also reduces the K-width of the system (i.e., $K_d^{max}$). This has been reported through an inequality plot that shows the inverse relationship between the actuator damping $b$ and the normalized maximum renderable stiffness $K_d^{max}/K$ ~\cite{Tagliamonte2014}.
	
	To maximize the K-width of the system, the velocity integral gain $I_m$ needs to be maximized. Our least conservative bounds allow $I_m$ to attain its maximum value without violating passivity; thus, enlarge the K-width of the system to its theoretical limit.
	
	Another important finding of this study reveals that the presence of damping necessitates an extra passivity constraint. If the actuator is modeled as pure inertia, that is, $b = 0$, the condition in Proposition~2 reduces to
	\begin{equation}
	J_{max}^{null}<\frac{P_{m}(1+P_{m}P_{t})}{P_{m}I_{t}+P_{t}I_{m}}
	\label{undamped}
	\end{equation}
	
	Hence, when physical damping is neglected in the system model such that the actuator is modeled as pure inertia, a necessary condition for passivity is missed. This result is counterintuitive in that increasing damping is typically expected to result in less conservative passivity conditions due to its dissipative nature. However, this intuition fails in the presence of integral controllers and introduction of physical actuator damping into the system model imposes an additional constraint to ensure passivity, instead of relaxing passivity conditions. Therefore, physical damping should not be neglected in the passivity analysis, especially if integrators are utilized. This result surprisingly demonstrates the adversary effect of physical damping on passivity.
	
	\begin{table}[b]
		\begin{center}
			\caption{Design Guidelines for Rendering Virtual Spring}
			\vspace{-.5\baselineskip}
			\resizebox{\columnwidth}{!}{\begin{tabular}{l|l}
					\hline
					Vallery~\emph{et al.}~\cite{Vallery2007,Vallery2008} & $P_{m}>J \wedge P_{m}>2I_{m} \wedge P_{t}>2I_{t} \wedge K_{d} < K_{d}^{max}|_{b=0}$  \\
					\hline
					&
					\\    		
					Accoto~\emph{et al.}~\cite{Tagliamonte2014} & $J<$  {\large $ \frac{(P_{m}+b)(P_{m}\:P_{t})} {P_{m}\:I_{t}+P_{t}\:I_{m}} $} $\wedge$ {\large $b<\frac{P_{t}I_{m}}{I_{t}}$} $\: \wedge \: K_{d}<K_{d}^{max}$
					\\
					&
					\\    			
					\hline
					&
					\\
					Ours & $J<$  {\large $\frac{(P_{m}+b)(\Delta K\:P_{m}\:P_{t}+K)}
						{\Delta K\:(P_{m}\:I_{t}+P_{t}\:I_{m})} $} $\wedge$ {\large $b<\frac{P_{t}I_{m}}{I_{t}}$} $\wedge \: K_{d}<K_{d}^{max}$ (Eqn.~(\ref{kd_max})) \\ 	\\		
					\hline
			\end{tabular}}
		\end{center}
	\end{table}
	
	To emphasize this fact, a numerical example is provided. Assume we have the SEA plant as given in Table~II. Two controllers are suggested: The first controller has been tuned according to Proposition~2, while the second controller has been tuned according to Eqn.~(\ref{undamped}). The numerical values for the control parameters used in the simulation are reported in Table~VII. In Section~IV, it has been shown that larger $I_{t}$ gains provide better rendering performance for null impedance. Hence, the largest possible values with a small safety margin have been chosen for $I_{t}$ for both systems. Note that when the damping is included in the actuator model, the upper bound for $I_{t}$ dramatically decreases because of the additional constraint introduced due to the presence of damping.
	
	Figure~\ref{plot:undamped} presents the Bode plots of these two systems. Note that, both systems are theoretically passive according to their respective actuator models that are with and without damping. However, to test the controllers, damping is included in the simulated actuator model of both systems, since some level of dissipation is always present in physical systems.
	
	\begin{figure}[t]
		\resizebox{\columnwidth}{!}{\includegraphics{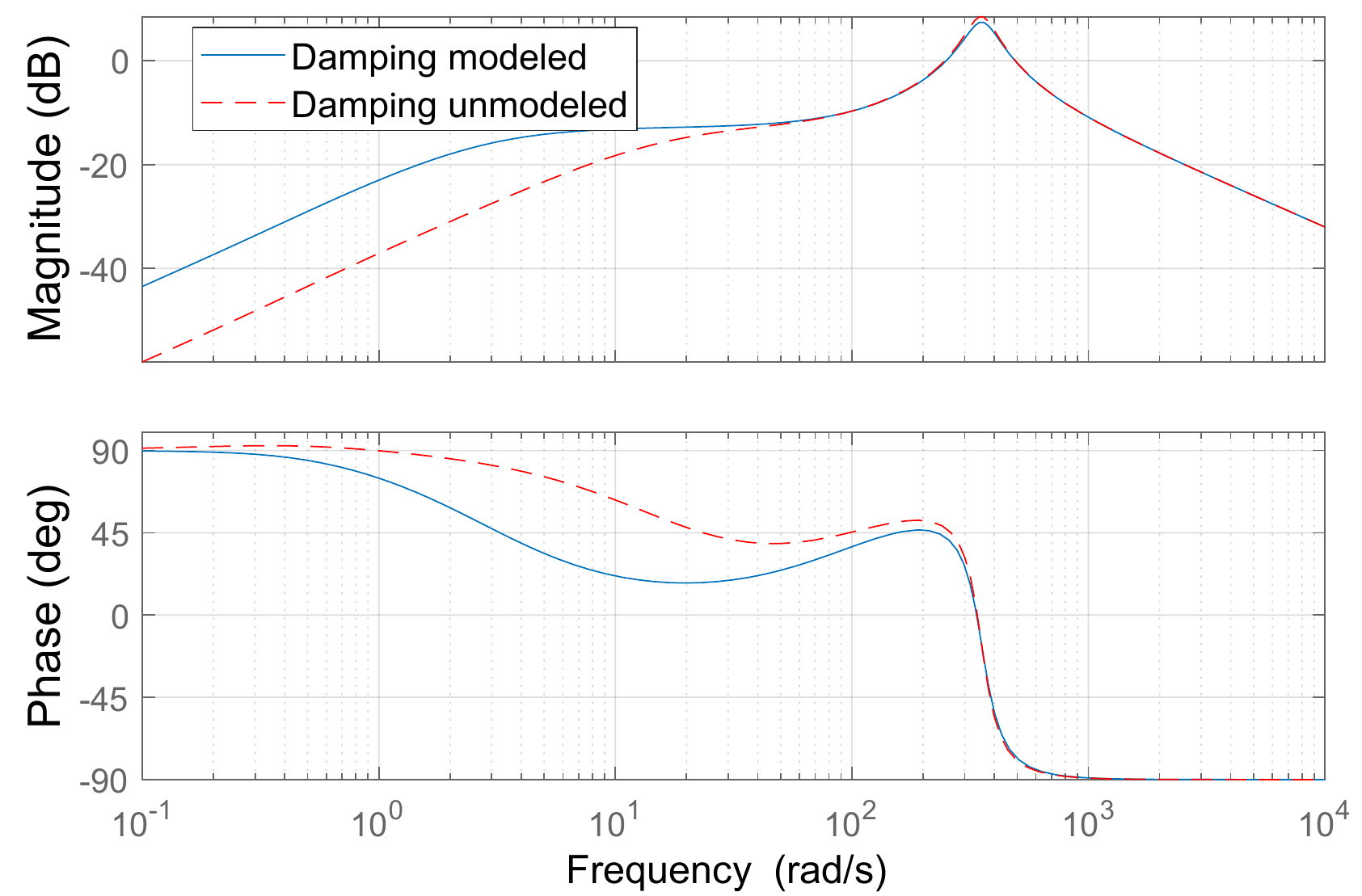}}
		\vspace{-1.5\baselineskip}
		\caption{The effect of damping on passivity}
		\label{plot:undamped}
		\vspace{-1.25\baselineskip}
	\end{figure}
	
	Clearly, the second controller outperforms the first one, but at the expense of passivity. Simulation results indicate that the phase of the second controller passes $90^\circ$  for a range of low frequencies and goes up to $93.5^\circ$. This result serves as a counter-example for the commonly employed assumption that neglecting damping results in more conservative passivity conditions.
	
	
	In fact, similar counter-examples that falsify the presumption that an addition of damping relaxes the passivity bounds have also been noted in the literature. In particular, a numerical parameter space search was used in~\cite{Newman2003} to analyse the passivity of Natural Admittance Control~\cite{Newman1992} and an adversary relationship between the integral control gain and the virtual damping parameters in the presence of physical damping has been noted. Similarly, in~\cite{Willaert2011}, the need for verifying passivity at the upper and lower bounds on the damping parameter has been advocated within the concept of bounded impedance passivity. Our results are in good agreement with these earlier observations and support them by proving the necessity of bounds on integral gains when physical damping of the system is included in the system model.

	\section{Conclusion}
	
	We have presented the necessary and sufficient conditions to ensure the passivity of cascade-control of SEA for rendering null impedance and pure stiffness models. These conditions extend the sufficiency condition reported in the literature~\cite{Vallery2007,Vallery2008,Tagliamonte2014} and relaxing these bounds, serve as the least conservative bounds on renderable impedances under the frequency domain passivity paradigm. Our results also prove the necessity of  a counter-intuitive second bound on integral gains, that has been neglected in the literature. This bound is crucial as it is imposed due to inevitable physical damping of the actuator; hence, cannot be safely neglected if integral controller is used in both the inner  and the intermediate control loops.
	
	While the necessary and sufficient conditions provide the least conservative bounds within the frequency domain passivity paradigm, they may still be conservative. Along these lines, less conservative paradigms, such as time domain passivity~\cite{hannafordRyu,ryuGeneral}, complementary stability~\cite{Buerger2007}, bounded-impedance absolute stability~\cite{Haddadi2010,Willaert2011}, fractional-order passivity~\cite{OzanIros,Yusuf2018} may be utilized to achieve better performance while still ensuring coupled stability of interaction. However, even though they are relatively conservative, frequency domain passivity conditions are valuable as they are known to provide a fundamental understanding of the underlying trade-offs governing the dynamics of the closed-loop system.

	\begin{table}[t]
		\caption{Controller gains for the two models}
		\begin{subtable}{.5\linewidth}
			\centering
			\vspace{-.5\baselineskip}
			\begin{tabular}{l|l}
				\hline
				\multicolumn{2}{c}{First Controller} \\ 
				\hline
				$P_{m}$ & 20 Nm s/rad \\
				$P_{t}$ & 5  rad/(s Nm) \\
				$I_{m}$ & 10 Nm/rad  \\
				$I_{t}$ & 15 rad/(s$^2$ Nm) \\
				\hline
			\end{tabular}
		\end{subtable}%
		\begin{subtable}{.5\linewidth}
			\centering
			\vspace{-.5\baselineskip}
			\begin{tabular}{l|l}
				\hline
				\multicolumn{2}{c}{Second Controller} \\ 
				\hline
				$P_{m}$ & 20 Nm s/rad \\
				$P_{t}$ & 5  rad/(s Nm) \\
				$I_{m}$ & 10 Nm/rad  \\
				$I_{t}$ & 80 rad/(s$^2$ Nm)  \\
				\hline
			\end{tabular}
		\end{subtable}\vspace{-1\baselineskip}
	\end{table} \label{Tab:two_controllers}
	
	\vspace{-.5\baselineskip}
	\bibliographystyle{IEEEtran}
	\bibliography{ref}

\end{document}